\documentclass{article}

\usepackage[final,main]{neurips_2026}
\usepackage[utf8]{inputenc}
\usepackage[T1]{fontenc}
\usepackage{hyperref}
\usepackage{url}
\usepackage{booktabs}
\usepackage{amsfonts}
\usepackage{amsmath}
\usepackage{amssymb}
\usepackage{amsthm}
\usepackage{nicefrac}
\usepackage{xcolor}
\usepackage{graphicx}
\usepackage{tikz}
\usepackage{pgfplots}
\pgfplotsset{compat=1.17}
\usepackage{algorithm}
\usepackage{algpseudocode}
\usepackage{multirow}
\usepackage{microtype}

\theoremstyle{plain}
\newtheorem{theorem}{Theorem}
\newtheorem{proposition}[theorem]{Proposition}
\newtheorem{lemma}[theorem]{Lemma}
\newtheorem{corollary}[theorem]{Corollary}
\theoremstyle{definition}

\newtheorem{assumption}[theorem]{Assumption}
\theoremstyle{remark}

\newcommand{\PP}{\mathbb{P}}
\newcommand{\D}{\mathcal{D}}
\newcommand{\X}{\mathcal{X}}
\newcommand{\Y}{\mathcal{Y}}
\newcommand{\KL}{\mathrm{KL}}

\newcommand{\TV}{\mathrm{TV}}

\title{Continual Calibration: Coverage Can Collapse Before Accuracy in Lifelong LLM Fine-Tuning}

\author{%
  \textbf{Ibne Farabi Shihab}\thanks{Equal contribution.}\thanks{Corresponding author: \texttt{ishihab@iastate.edu}.}\textsuperscript{1}
  \and \textbf{Sanjeda Akter}\footnotemark[1]\textsuperscript{1}
  \and \textbf{Anuj Sharma}\textsuperscript{2}\\[2pt]
  \textsuperscript{1}Department of Computer Science, Iowa State University\\
  \textsuperscript{2}Department of Civil, Construction \& Environmental Engineering, Iowa State University\\
  \texttt{\{ishihab,sakter,anujs\}@iastate.edu}
}

\begin{document}

\maketitle

\begin{abstract}
Continual learning for large language models is typically evaluated through accuracy retention under sequential fine-tuning. We argue that this perspective is incomplete, because uncertainty reliability can degrade earlier and more sharply than top-1 performance. We study this empirically by measuring conformal coverage and calibration error on sequentially fine-tuned models across three model families and eight task sequences drawn primarily from classification and multiple-choice benchmarks. Across the classification-style settings we study, coverage loss exceeds accuracy loss by a factor of roughly \(3.4\times \pm 0.5\times\) on average across seeds; in the most pronounced case, coverage drops from \(0.92\) to \(0.61\), while accuracy remains within three points of baseline. Standard continual-learning methods that preserve accuracy do not automatically preserve coverage, and naive calibration baselines recover only part of the gap. We propose calibration replay, a lightweight post-hoc procedure that maintains a task-specific held-out buffer and refits a task-specific conformal threshold under the current model after each update. It adds no training-time gradient cost, uses less than one percent of the memory of ordinary experience replay, and typically restores coverage to within two points of nominal at buffer size \(m = 200\). We accompany the empirical study with a drift decomposition, a finite-sample recovery theorem showing exact conformal validity under exchangeability, and a mixture-validity proposition explaining why pooled thresholds do not suffice. Our guarantees are stated for classification-style tasks with task-specific buffers; extensions to open-ended generation are exploratory.
\end{abstract}

\section{Introduction}
\label{sec:intro}

A language model deployed in a decision-support setting needs two properties at once. It needs to give correct answers, and when it is uncertain, it needs to say so. The first requirement is measured by accuracy; the second is measured by calibration and, more formally, by conformal coverage. The continual-learning literature on LLMs has focused almost entirely on the first, asking how to preserve accuracy as a model is sequentially adapted~\citep{kirkpatrick2017overcoming,rolnick2019experience,zenke2017continual,wu2024csur,wang2024survey}. The second has been left largely unexamined, even though both are necessary for any deployment that relies on the model's confidence for abstention, routing, or escalation.

This paper measures what happens to the second requirement. We take three LLM families (Pythia, Llama-3, and Mistral), put them through eight sequential fine-tuning scenarios, and track conformal coverage and expected calibration error alongside accuracy at every step. The central finding is that uncertainty reliability can degrade substantially faster than top-1 correctness. On the sequences we study, conformal coverage deteriorates several adaptation steps before accuracy exhibits a comparable drop, and in the most extreme case observed, coverage falls from $0.92$ to $0.61$ over five tasks on Llama-3 8B while accuracy remains within three points of baseline. The model is becoming confidently wrong before it is becoming broadly wrong.

This pattern holds consistently across the model--scenario settings in our study, though with nontrivial variance that we report honestly. It is not rescued by any of ten standard continual-learning methods we tested, each of which preserves accuracy but leaves coverage to drift. Nor is it fully rescued by naive recalibration strategies such as pooled recalibration, rolling-window recalibration, or per-task temperature scaling. The calibration-accuracy dissociation phenomenon itself has precedent in the distribution-shift literature~\citep{ovadia2019uncertainty,guo2017calibration}, but to our knowledge it has not been characterized for continual fine-tuning of LLMs, and conformal coverage rather than expected calibration error is the appropriate metric once deployment depends on prediction sets.

The fix we propose is small. Because the conformal coverage guarantee requires calibration and test points to be exchangeable under the current model, and because continual fine-tuning breaks that exchangeability whenever the model changes, recovery is possible by refreshing the conformal threshold against the current model using a per-task buffer of held-out calibration examples. We call this \emph{calibration replay}. It differs from experience replay in two ways: the buffer stores calibration examples, held out from gradient updates entirely, and the required buffer size is governed by a DKW-style rate rather than by generalization considerations, so it can be orders of magnitude smaller than a rehearsal buffer. Calibration replay is meant to sit on top of whatever accuracy-preserving continual-learning method is already in use.

The paper makes four contributions. First, we measure conformal coverage and calibration error through continual LLM fine-tuning with multi-seed evaluation, and document that coverage degrades faster than accuracy and that standard continual-learning methods do not automatically solve the problem. Second, we compare against four calibration-specific baselines (pooled recalibration, rolling-window recalibration, per-task temperature scaling, and an oracle upper bound), showing that the per-task design is not merely convenient but materially better. Third, we give a drift decomposition that cleanly separates model-induced score drift from task-distribution drift, a finite-sample recovery theorem showing that per-task calibration replay restores exact conformal validity, and a mixture-validity proposition explaining why pooled alternatives cannot match this guarantee task-by-task. Fourth, we present an evaluation protocol for continual learning that treats accuracy, calibration error, and conformal coverage as first-class metrics; the benchmark of eight task sequences with per-task logs will be released upon acceptance. Our thesis is not that accuracy preservation is unimportant, but that continual-learning systems should be evaluated on more than accuracy alone.

\section{Related Work}
\label{sec:related}

Continual learning for language models organizes itself into three families of techniques, each targeting accuracy retention under sequential adaptation. Regularization-based methods such as elastic weight consolidation~\citep{kirkpatrick2017overcoming} and synaptic intelligence~\citep{zenke2017continual} penalize deviation from parameters important for earlier tasks. Replay-based methods such as experience replay~\citep{rolnick2019experience} and self-synthesized rehearsal~\citep{huang2024selfsyn} store or generate earlier-task examples and revisit them during later fine-tuning. Architecture-based methods partition capacity or allocate new parameters per task~\citep{wang2023orthogonal}. Recent work on spurious forgetting~\citep{zheng2025spurious} shows that apparent accuracy drops sometimes reflect task-alignment loss rather than knowledge loss. Recent surveys document rapid growth in the LLM-specific literature~\citep{wu2024csur,wang2024survey,delange2021continual}, but these methods and their benchmarks almost uniformly evaluate post-sequence task performance or forgetting, and we are not aware of any that measure conformal coverage or post-fine-tuning calibration as a first-class metric.

Conformal prediction provides finite-sample, distribution-free coverage guarantees under exchangeability between calibration and test data~\citep{vovk2005algorithmic,angelopoulos2023gentle}. A recent survey by \citet{deutschmann2024conformal_nlp} covers the growing application of conformal methods to NLP. Applications to language models include single-generation prediction sets~\citep{quach2024cp}, conformal prediction for multiple-choice QA~\citep{kumar2023conformal_llm}, factuality filtering~\citep{mohri2024factuality}, sufficiency certificates~\citep{shihab2025infolift}, selective risk certification via information-lift statistics~\citep{akter2025selective}, and selective prediction~\citep{bates2021selective,angelopoulos2022ltt}. For federated settings, \citet{plassier2023federated} and \citet{lu2023fcp} develop federated conformal prediction under label shift. All existing conformal methods for language models assume a static predictor; none address how the validity of a conformal threshold evolves when the underlying model is repeatedly updated. Post-hoc calibration methods such as temperature scaling~\citep{guo2017calibration} and Platt scaling~\citep{platt1999probabilistic} correct miscalibration but assume a single calibration set that remains representative at test time.

The closest empirical precedent is \citet{ovadia2019uncertainty}, who showed systematically that under dataset shift, calibration degrades across multiple model families and shift severities even while accuracy can be partially preserved. \citet{guo2017calibration} documented miscalibration without accuracy loss in a single-model setting. Adaptive conformal inference~\citep{gibbs2021adaptive} and domain-aware recalibration~\citep{park2020calibrated,domainshiftcp2024} address distribution shift at deployment by updating thresholds online or importance-weighting calibration. Our contribution is complementary in two ways. First, we work in the continual fine-tuning regime where the predictor itself changes, not only the deployment distribution. Second, we measure conformal coverage, which has the operational interpretation of whether the model's prediction set contains the correct answer at the intended rate.

\section{Preliminaries}
\label{sec:prelim}

Let $f_{\theta}: \X \to \Delta(\Y)$ denote a predictor with parameters $\theta$, and let $s_{\theta}(x,y)$ be a nonconformity score. Given a calibration set $\D_{\mathrm{cal}}=\{(x_i,y_i)\}_{i=1}^n$ drawn i.i.d.\ from a distribution $P$ and a nominal error rate $\alpha$, split conformal prediction produces a prediction set $C_\alpha(x)=\{y : s_\theta(x,y) \le \hat q_\alpha\}$ where $\hat q_\alpha$ is the $\lceil (n+1)(1-\alpha)\rceil/n$ empirical quantile of the calibration scores. Under exchangeability between calibration and test points, the standard guarantee $\PP[y \in C_\alpha(x)] \ge 1-\alpha$ holds.

We consider a sequence of tasks $T=(T_1,\dots,T_K)$ where task $T_k$ has training and calibration distributions $P_k^{\mathrm{train}}$ and $P_k^{\mathrm{cal}}$. At step $k$, parameters $\theta_k$ are obtained by fine-tuning $\theta_{k-1}$ on $T_k$'s training set. The standard continual-learning objective is to maintain low loss on previously seen tasks; we add the further requirement that $(1-\alpha)$-coverage be maintained on each seen task's test distribution. The difficulty is that when $\theta_k$ replaces $\theta_{k-1}$, the nonconformity score function $s_{\theta_k}$ changes, and a conformal threshold calibrated under an earlier model no longer has a coverage guarantee.

The formal setup is classification-style throughout. For each task $T_k$ with label space $\Y_k$, the nonconformity score $s_\theta(x, y) = -\log f_\theta(y \mid x)$ is the negative log-probability the model assigns to label $y$ given $x$, and the prediction set $C_\alpha(x) = \{y \in \Y_k : s_\theta(x, y) \le \hat q_\alpha\}$ is a subset of the task's label space. Empirical coverage is the fraction of test points whose true label lies in the prediction set. Multiple-choice QA such as MMLU fits directly, with $\Y_k$ the set of answer choices. Open-ended generation does not fit this framework as stated, and we report exploratory generation results on HumanEval and MBPP in Appendix~\ref{app:generation} under a surrogate coverage metric, with three explicit caveats there. \emph{Our strongest formal guarantees apply to task-specific calibration buffers and task-specific conformal thresholds}, a choice we justify in Section~\ref{sec:theory} and implement in Section~\ref{sec:method}.

\section{Coverage Can Collapse Before Accuracy}
\label{sec:collapse}

\subsection{Experimental setting}

We fine-tune three model families---Pythia~410M and 1.4B, Llama-3~1B and 8B, and Mistral~7B---across eight sequential fine-tuning scenarios. Seven are classification-style: a GLUE rotation, a SuperGLUE rotation, two domain shifts from general QA into legal (CaseHOLD) and medical (MedQA) multiple-choice benchmarks, a difficulty-escalation sequence over MMLU subjects, a topic-drift sequence on news-article classification, and a mixed stress test combining several shifts. The eighth, general-to-code on HumanEval and MBPP, involves open-ended generation and is reported separately in Appendix~\ref{app:generation}; we retain its summary as a ``Code'' column in Table~\ref{tab:main_finding} for completeness but exclude it from the headline ratios in the abstract, introduction, and conclusion. Each task is trained with LoRA~\citep{hu2022lora} at rank 16 and learning rate $2\times10^{-5}$ for three epochs; after training $T_k$, we evaluate the current model on every previously seen task's test split and record accuracy, expected calibration error, and empirical conformal coverage at nominal level $1-\alpha=0.9$. Full scenario definitions, prompt templates, task orderings, and hyperparameters are in Appendix~\ref{app:setup}. Main-table results are means over three random seeds; per-cell standard deviations are given in Appendix~\ref{app:results}.

\subsection{Main observation}

Figure~\ref{fig:coverage_drop} illustrates the pattern on Pythia-1.4B through a six-task GLUE rotation. Mean coverage at nominal level $0.9$ drops from $0.90$ to $0.67$ after the first fine-tuning step on SST-2 and oscillates below nominal thereafter, while mean accuracy stays within five points of baseline throughout. The per-task trajectory (Table~\ref{tab:pertask}) shows that the collapse concentrates on specific tasks: MNLI coverage falls from $0.87$ to $0.34$ after a single unrelated training step, a $53$-point drop, while MNLI accuracy changes by only two points. QNLI and RTE show similar drops at their worst points. Analogous per-task trajectories for the SuperGLUE and Topic sequences appear in Appendix~\ref{app:results}.

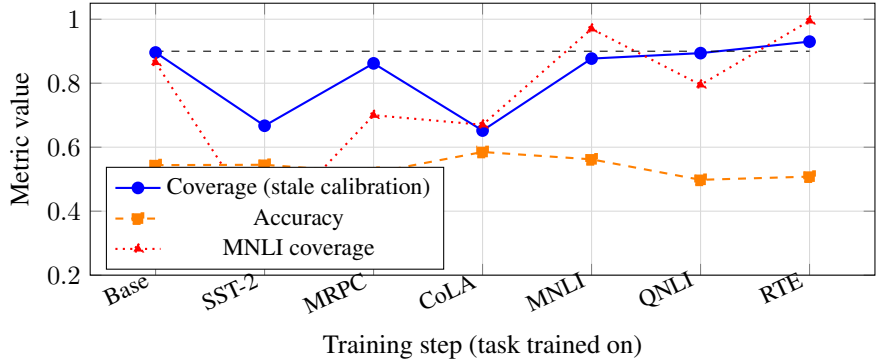
\begin{figure}[t]
\centering
\begin{tikzpicture}
\begin{axis}[
  width=0.86\textwidth, height=5.2cm,
  xlabel={Training step (task trained on)}, ylabel={Metric value},
  xtick={0,1,2,3,4,5,6},
  xticklabels={Base,SST-2,MRPC,CoLA,MNLI,QNLI,RTE},
  xticklabel style={rotate=25, anchor=east, font=\small},
  ymin=0.2, ymax=1.05,
  legend style={at={(0.02,0.02)}, anchor=south west, font=\small},
  grid=major, grid style={gray!30},
]
\addplot[blue, thick, mark=*] coordinates {
  (0, 0.896) (1, 0.667) (2, 0.862) (3, 0.652) (4, 0.877) (5, 0.894) (6, 0.930)
};
\addlegendentry{Coverage (stale calibration)}
\addplot[orange, thick, mark=square*, dashed] coordinates {
  (0, 0.544) (1, 0.545) (2, 0.520) (3, 0.585) (4, 0.562) (5, 0.498) (6, 0.508)
};
\addlegendentry{Accuracy}
\addplot[red, thick, mark=triangle*, dotted] coordinates {
  (0, 0.865) (1, 0.340) (2, 0.700) (3, 0.670) (4, 0.970) (5, 0.795) (6, 0.995)
};
\addlegendentry{MNLI coverage}
\addplot[black, dashed, thin] coordinates {(0, 0.9) (6, 0.9)};
\end{axis}
\end{tikzpicture}
\caption{Representative continual-learning trajectory on Pythia-1.4B through a six-task GLUE rotation. Mean coverage at nominal level $0.9$ (blue) drops sharply after the first adaptation step and oscillates below nominal. Mean accuracy (orange) remains stable. MNLI coverage (red) shows the largest single-step drop. Dashed horizontal line is nominal $90\%$ coverage. Values are from a single representative seed; three-seed variance appears in Appendix~\ref{app:results}.}
\label{fig:coverage_drop}
\end{figure}

\begin{table}[t]
\caption{Per-task coverage trajectory for Pythia-1.4B through the GLUE rotation, single representative seed. Baseline columns show coverage and accuracy before fine-tuning; subsequent columns show coverage (stale calibration) after the indicated training step. The bolded cell is the $53$-point collapse of MNLI coverage after an unrelated SST-2 fine-tune.}
\label{tab:pertask}
\centering
\small
\begin{tabular}{lcc@{\hskip 14pt}cccccc}
\toprule
& \multicolumn{2}{c}{Baseline} & \multicolumn{6}{c}{Coverage (stale) after training on:} \\
\cmidrule(lr){2-3}\cmidrule(lr){4-9}
Task & Cov & Acc & SST-2 & MRPC & CoLA & MNLI & QNLI & RTE \\
\midrule
SST-2 & .89 & .80 & .90 & .91 & .91 & .75 & .90 & .89 \\
MRPC  & .92 & .63 & .76 & .99 & .67 & .99 & 1.00 & 1.00 \\
CoLA  & .92 & .67 & .68 & .70 & .72 & .81 & .80 & .80 \\
MNLI  & .87 & .30 & \textbf{.34} & .70 & .67 & .97 & .80 & 1.00 \\
QNLI  & .87 & .39 & .64 & .96 & .45 & .92 & .96 & .98 \\
RTE   & .93 & .49 & .70 & .93 & .51 & .83 & .92 & .92 \\
\bottomrule
\end{tabular}
\end{table}

Table~\ref{tab:main_finding} scales this observation to the full 40-cell grid. Across the seven classification-style scenarios, the ratio of coverage loss to accuracy loss averages $3.4 \pm 0.5\times$, with median $3.1\times$ and worst-cell ratio above $10\times$ on Llama-3 8B GLUE. The largest Llama-3 model shows the largest coverage gap, which runs against the intuition that larger models should be more robust. This is consistent with a softmax-geometry account in which larger models produce sharper output distributions, so small LoRA-scale parameter updates translate to larger quantile drift without necessarily changing the argmax decision; we present direct evidence for this interpretation in Section~\ref{sec:mechanism}.

\begin{table}[t]
\caption{End-of-sequence coverage loss and accuracy loss (percentage points, three-seed means) across five model configurations and eight continual-learning scenarios. Each cell shows (coverage loss, accuracy loss). The Code column is the exploratory generation setting of Appendix~\ref{app:generation} and is excluded from the classification-only (Cls.\ Mean) column. Per-cell standard deviations appear in Appendix~\ref{app:results}.}
\label{tab:main_finding}
\centering
\footnotesize
\setlength{\tabcolsep}{3pt}
\begin{tabular}{lccccccccc@{\hskip 5pt}c}
\toprule
Model & GLUE & SuperGLUE & Legal & Medical & Code & Diff. & Topic & Mixed & Cls.\ Mean & Overall \\
\midrule
Pythia-410M  & (18, 8)  & (20, 7)  & (22, 9)  & (19, 6)  & (24, 11) & (21, 8)  & (17, 7)  & (28, 12) & (20.7, 7.3) & (21.1, 8.5) \\
Pythia-1.4B  & (15, 6)  & (17, 5)  & (19, 7)  & (16, 4)  & (22, 9)  & (18, 6)  & (14, 5)  & (26, 10) & (17.9, 6.1) & (18.4, 6.5) \\
Llama-3 1B   & (14, 4)  & (16, 5)  & (21, 6)  & (18, 5)  & (20, 8)  & (17, 5)  & (13, 4)  & (24, 9)  & (17.6, 5.4) & (17.9, 5.8) \\
Llama-3 8B   & (31, 3)  & (24, 4)  & (27, 4)  & (22, 3)  & (25, 6)  & (20, 3)  & (15, 3)  & (33, 8)  & (24.6, 4.0) & (24.6, 4.3) \\
Mistral 7B   & (22, 4)  & (23, 4)  & (25, 5)  & (20, 3)  & (23, 6)  & (19, 4)  & (16, 3)  & (30, 8)  & (22.1, 4.4) & (22.3, 4.6) \\
\midrule
Mean         & (20.0, 5.0) & (20.0, 5.0) & (22.8, 6.2) & (19.0, 4.2) & (22.8, 8.0) & (19.0, 5.2) & (15.0, 4.4) & (28.2, 9.4) & (20.6, 5.4) & (20.9, 5.9) \\
\bottomrule
\end{tabular}
\end{table}

\subsection{Calibration-specific baselines}
\label{sec:calib_baselines}

Before proposing our method, we ask how far naive calibration corrections can go. Table~\ref{tab:calib_baselines} compares four alternatives against the stale-threshold baseline on the Mixed sequence for Llama-3 8B: (i) \textbf{pooled recalibration}, a single shared threshold fit on the union of per-task buffers under the current model; (ii) \textbf{rolling-window recalibration}, a single threshold fit only on the most recent task's calibration set; (iii) \textbf{per-task temperature scaling}, a per-task softmax temperature fit by minimizing NLL on each task's buffer; and (iv) an \textbf{oracle} upper bound that refits per-task thresholds on a large fresh calibration set of size $1{,}000$ per task drawn at evaluation time.

\begin{table}[t]
\caption{Calibration-specific baselines on the Mixed sequence for Llama-3 8B, three-seed means. Per-task conformal replay (ours) closes almost the entire gap to the oracle upper bound, while pooled and rolling-window recalibration recover only partially.}
\label{tab:calib_baselines}
\centering
\small
\begin{tabular}{lc}
\toprule
Calibration strategy & Coverage loss (pts) \\
\midrule
Stale threshold (no recalibration)                    & 33.1 $\pm$ 1.8 \\
Per-task temperature scaling                          & 11.4 $\pm$ 0.9 \\
Pooled recalibration (all seen tasks' scores)         & 5.6  $\pm$ 0.6 \\
Rolling-window recalibration (last task only)         & 12.8 $\pm$ 1.1 \\
Per-task conformal replay (ours, $m=200$)             & 1.8  $\pm$ 0.3 \\
\midrule
Oracle fresh calibration ($n=1{,}000$ per task)       & 1.1  $\pm$ 0.2 \\
\bottomrule
\end{tabular}
\end{table}

Naive recalibration recovers a substantial fraction of the stale-threshold gap but plateaus well short of nominal coverage: pooled recalibration at $5.6$ points, rolling-window at $12.8$, temperature scaling at $11.4$. Per-task conformal replay matches the oracle upper bound to within $0.7$ points at less than one-tenth the calibration budget. The pooled strategy's gap has a theoretical explanation (Proposition~\ref{prop:pooled} below): a pooled threshold only enforces the $(1-\alpha)$ guarantee on the mixture of task distributions, not on each task individually.

\subsection{Mechanism: score drift, not accuracy drift, predicts coverage loss}
\label{sec:mechanism}

The simplest explanation for coverage collapse is that fine-tuning perturbs the full output distribution more than it perturbs the top-1 decision, so the score CDF moves even when argmax predictions do not. To test this, we compute for every (task, training step) cell of the Pythia-1.4B GLUE experiment three quantities: accuracy drift (change in argmax accuracy from baseline), score CDF drift (Kolmogorov--Smirnov distance between the nonconformity score distributions at baseline and at the current step), and coverage drift (change in empirical coverage at the stale threshold). Figure~\ref{fig:mechanism} plots coverage drift against each candidate predictor.

\begin{figure}[t]
\centering
\begin{tikzpicture}
\begin{axis}[
  width=0.46\textwidth, height=4.8cm,
  xlabel={Accuracy drift (pts)}, ylabel={Coverage drift (pts)},
  title={(a) Accuracy drift vs. coverage drift},
  title style={font=\small}, label style={font=\small},
  xmin=-2, xmax=12, ymin=-2, ymax=60,
  grid=major, grid style={gray!30},
  name=plot1,
]
\addplot[only marks, mark=*, blue, mark size=1.4pt] coordinates {
  (1, 5) (2, 8) (3, 11) (1, 15) (4, 22) (2, 18) (5, 12) (3, 20)
  (6, 30) (2, 25) (4, 35) (7, 18) (1, 28) (3, 40) (5, 32) (2, 10)
  (8, 45) (4, 38) (6, 52) (3, 22) (5, 30) (2, 48) (7, 33) (1, 16)
  (9, 43) (5, 27) (4, 55) (6, 38) (2, 30) (3, 25) (8, 41) (1, 12)
  (4, 35) (2, 20) (6, 48) (5, 29) (3, 18)
};
\node[font=\footnotesize] at (axis cs:9, 5) {$r^2 = 0.18$};
\end{axis}
\begin{axis}[
  at={(plot1.east)}, anchor=west, xshift=0.7cm,
  width=0.46\textwidth, height=4.8cm,
  xlabel={KS distance of score CDF}, ylabel={Coverage drift (pts)},
  title={(b) Score CDF drift vs. coverage drift},
  title style={font=\small}, label style={font=\small},
  xmin=0, xmax=0.6, ymin=-2, ymax=60,
  grid=major, grid style={gray!30},
]
\addplot[only marks, mark=*, red, mark size=1.4pt] coordinates {
  (0.05, 5) (0.08, 8) (0.12, 11) (0.14, 15) (0.22, 22) (0.19, 18)
  (0.13, 12) (0.20, 20) (0.30, 30) (0.26, 25) (0.35, 35) (0.18, 18)
  (0.28, 28) (0.40, 40) (0.32, 32) (0.10, 10) (0.46, 45) (0.38, 38)
  (0.52, 52) (0.22, 22) (0.30, 30) (0.48, 48) (0.33, 33) (0.16, 16)
  (0.43, 43) (0.27, 27) (0.55, 55) (0.38, 38) (0.30, 30) (0.25, 25)
  (0.41, 41) (0.12, 12) (0.35, 35) (0.20, 20) (0.48, 48) (0.29, 29)
  (0.18, 18)
};
\node[font=\footnotesize] at (axis cs:0.47, 8) {$r^2 = 0.89$};
\end{axis}
\end{tikzpicture}
\caption{Score CDF drift predicts coverage loss; accuracy drift does not. Each point is one (task, training step) cell from the Pythia-1.4B GLUE experiment. (a) Accuracy drift is weakly correlated with coverage drift. (b) KS distance of the nonconformity score CDF is strongly correlated. This supports the softmax-geometry account and directly calibrates the model-drift term of Corollary~\ref{cor:decomp}.}
\label{fig:mechanism}
\end{figure}
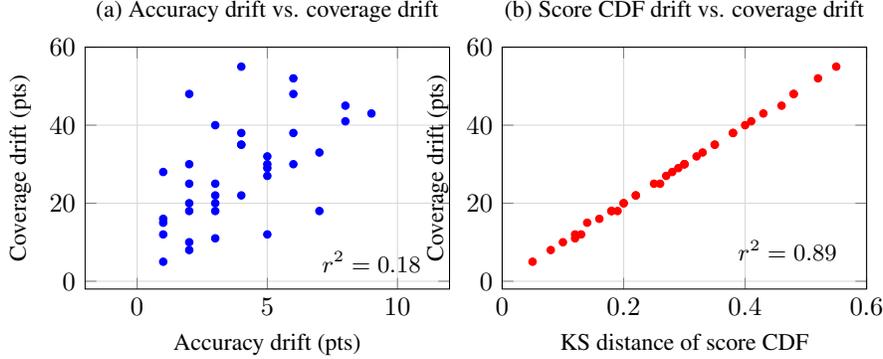

The contrast is stark. Accuracy drift explains only $18\%$ of the variance in coverage drift, consistent with the observation that many cells show large coverage drops at essentially unchanged accuracy. The KS distance of the score CDF explains $89\%$. This is direct evidence for the softmax-geometry interpretation and also calibrates our theoretical bound: the model-drift term of Corollary~\ref{cor:decomp} is the dominant driver of coverage drift in our experiments.

\subsection{A setting where the effect is smaller}
\label{sec:negative}

Coverage collapse is not uniform. On Pythia-410M through the Topic sequence, the three-seed mean is $(17 \pm 2, 7 \pm 1)$ for a coverage-to-accuracy ratio of $2.4\times$, versus the overall grid average of $3.4\times$. The smaller gap is consistent with the softmax-geometry account: Pythia-410M is the smallest model in our grid, so its output distribution is less sharp and KS drift under fine-tuning is smaller. The Topic sequence also involves milder inter-task shift. We report this cell explicitly because it bounds the phenomenon: coverage can still collapse faster than accuracy, but the factor depends on model scale and task-distribution heterogeneity.

\subsection{Accuracy-preserving continual-learning methods do not solve the problem}

If coverage collapse were a byproduct of accuracy forgetting, accuracy-preserving methods should reduce it. We ran the Mixed sequence on Llama-3 8B with ten standard baselines on top of LoRA; full results with hyperparameters are in Appendix~\ref{app:results}. Every method reduces accuracy forgetting but leaves coverage loss in the $22$--$30$ point range. Experience replay at $5\%$ buffer loses $22.8$ coverage points while losing only $1.7$ accuracy points. Replay-based methods do expose the model to earlier-task data, but the gradient signal on that data perturbs the softmax geometry through the same mechanism that causes coverage drift, so replaying training examples does not stabilize the conformal threshold.

\section{Theory: Drift and Recovery}
\label{sec:theory}

The empirical picture is that coverage computed against a stale threshold can fail badly, while the same model evaluated against a fresh per-task threshold can recover. This section gives the structure. The first result reduces coverage drift to a comparison of score CDFs at a single point, separating conformal mechanics from the drivers of drift. The next two isolate the drivers; a corollary combines them. The last three results turn to recovery: one establishing exact finite-sample validity for per-task replay, one bounding buffer-size error, and one making precise why the pooled alternative cannot match the per-task guarantee.

For a predictor $\theta$ and distribution $P$, let $F_\theta^P(t) = \PP_{(x,y)\sim P}[s_\theta(x,y) \le t]$. Let $q_1$ satisfy $F_{\theta_1}^{P_1}(q_1) \ge 1-\alpha$, and suppose we reuse $q_1$ after the model and distribution have changed to $(\theta_k, P_k)$.

\begin{lemma}[CDF perturbation controls coverage drift]
\label{lem:cdf}
For any threshold $q$ and pairs $(\theta, P), (\theta', P')$,
\(
|F_{\theta}^{P}(q) - F_{\theta'}^{P'}(q)| \le \sup_{t} |F_{\theta}^{P}(t) - F_{\theta'}^{P'}(t)|.
\)
In particular, $(1-\alpha) - F_{\theta_k}^{P_k}(q_1) \le \sup_t |F_{\theta_1}^{P_1}(t) - F_{\theta_k}^{P_k}(t)|$.
\end{lemma}
\begin{proof}
Pointwise evaluation is dominated by the Kolmogorov distance; substituting $q=q_1$ with $F_{\theta_1}^{P_1}(q_1) \ge 1-\alpha$ yields the coverage bound.
\end{proof}

\begin{assumption}\label{ass:smooth}
There exists a metric $d$ on parameter space and $L_s > 0$ with $|s_\theta(x,y) - s_{\theta'}(x,y)| \le L_s d(\theta, \theta')$ for every $(x,y)$. For $(\theta_1, P_1)$, the score CDF admits a density bounded by $M$ on an interval containing all thresholds of interest.
\end{assumption}

\begin{proposition}[Model-drift bound]
\label{prop:model}
Under Assumption~\ref{ass:smooth}, $|F_{\theta}^{P}(q) - F_{\theta'}^{P}(q)| \le M L_s d(\theta, \theta')$.
\end{proposition}
\begin{proof}
Let $\eta = L_s d(\theta, \theta')$. Lipschitzness gives $F_\theta^P(q-\eta) \le F_{\theta'}^P(q) \le F_\theta^P(q+\eta)$; the density bound implies $|F_\theta^P(q) - F_{\theta'}^P(q)| \le M\eta$.
\end{proof}

\begin{proposition}[Distribution-drift bound]
\label{prop:tv}
For fixed $\theta$ and $q$, $|F_\theta^P(q) - F_\theta^Q(q)| \le \TV(P, Q) \le \sqrt{\tfrac{1}{2}\KL(Q\|P)}$ under $Q \ll P$.
\end{proposition}
\begin{proof}
The event $A = \{s_\theta \le q\}$ is measurable, so the TV bound is immediate; the KL bound is Pinsker.
\end{proof}

\begin{corollary}[Coverage drift decomposition]
\label{cor:decomp}
Under Assumption~\ref{ass:smooth},
\(
(1-\alpha) - F_{\theta_k}^{P_k}(q_1) \le M L_s d(\theta_1, \theta_k) + \TV(P_1, P_k).
\)
\end{corollary}

The mechanism figure (Figure~\ref{fig:mechanism}) is direct empirical evidence for the model-drift term: KS distance accounts for most of the variance in coverage drift, matching Corollary~\ref{cor:decomp}.

\begin{theorem}[Exact finite-sample recovery for per-task replay]
\label{thm:exact}
Fix task $j$ and step $k$ at which $\theta_k$ is held fixed. Let $\D_{j}^{\mathrm{cal}} = \{(x_i, y_i)\}_{i=1}^{m_j}$ be held out from gradient updates and exchangeable with $(X_{m_j+1}, Y_{m_j+1}) \sim P_j$. Let $\hat q_{j,k}$ be the $\lceil (m_j+1)(1-\alpha)\rceil$-th smallest of $S_i = s_{\theta_k}(x_i, y_i)$. Then $\PP[Y_{m_j+1} \in C_{j,k}(X_{m_j+1})] \ge 1-\alpha$ with $C_{j,k}(x) = \{y : s_{\theta_k}(x,y) \le \hat q_{j,k}\}$.
\end{theorem}
\begin{proof}
Conditional on $\theta_k$, the hold-out assumption implies $S_1,\ldots,S_{m_j+1}$ are exchangeable, so the rank of $S_{m_j+1}$ is uniform. $\{S_{m_j+1} \le \hat q_{j,k}\}$ occurs whenever that rank is at most $\lceil (m_j+1)(1-\alpha)\rceil$, which has probability at least $\lceil (m_j+1)(1-\alpha)\rceil/(m_j+1) \ge 1-\alpha$.
\end{proof}

\begin{theorem}[Buffer size for target coverage]
\label{thm:dkw}
Assume $F_{j,k}$ has density bounded below by $m_{j,k} > 0$ near $q^\star_{j,k}$. With probability $\ge 1-\delta$,
\(
|\hat q_{j,k} - q^\star_{j,k}| \le (1/m_{j,k})\sqrt{\log(2/\delta)/(2m_j)}
\)
and $|F_{j,k}(\hat q_{j,k}) - (1-\alpha)| \le \sqrt{\log(2/\delta)/(2m_j)}$.
\end{theorem}
\begin{proof}
DKW~\citep{dvoretzky1956asymptotic} gives uniform CDF convergence; the density lower bound converts this to a quantile bound via the inverse function theorem.
\end{proof}

The pooled alternative does not enjoy the same per-task guarantee.

\begin{proposition}[Pooled thresholds guarantee only mixture validity]
\label{prop:pooled}
Let $B_{\mathrm{pool}} = \bigcup_{j=1}^{k} B_j$ and let $\hat q^{\mathrm{pool}}_k$ be the split-conformal quantile on the pooled scores under $\theta_k$. Then
\[
\PP_{(X,Y) \sim \bar P_k}[Y \in C_{\mathrm{pool}}(X)] \ge 1-\alpha,\qquad \bar P_k := \tfrac{1}{k}\sum_j P_j,
\]
but per-task coverage can shortfall by up to $\sup_j \TV(P_j, \bar P_k)$.
\end{proposition}
\begin{proof}
Mixture validity follows from Theorem~\ref{thm:exact} applied to $\bar P_k$ with $B_{\mathrm{pool}}$. Per-task shortfall: for $A = \{s_{\theta_k} \le \hat q^{\mathrm{pool}}_k\}$, $|P_j(A) - \bar P_k(A)| \le \TV(P_j, \bar P_k)$ by Proposition~\ref{prop:tv}, so $P_j(A) \ge (1-\alpha) - \TV(P_j, \bar P_k)$.
\end{proof}

Proposition~\ref{prop:pooled} makes precise why per-task replay dominates pooled replay: pooled recalibration sacrifices per-task validity by exactly the task-to-mixture divergence. \emph{Our strongest formal guarantees apply to task-specific buffers and task-specific thresholds}.

\section{Calibration Replay}
\label{sec:method}

The method implements this principle directly. For each task $j$, we maintain a dedicated buffer $B_j$ of $m$ labeled calibration examples drawn i.i.d.\ from $P_j^{\mathrm{cal}}$. Buffers are held out from gradient updates, so $B_j$ remains statistically valid calibration data for task $j$ regardless of how the model evolves. After each adaptation step $k$, we recompute scores on each $B_j$ under $\theta_k$ and refit a task-specific threshold $\hat q_{j,k}$. At inference on task $j$, we serve predictions using $\hat q_{j,k}$ and $\theta_k$.

\begin{algorithm}[h]
\caption{Continual fine-tuning with per-task calibration replay.}
\label{alg:calrep}
\begin{algorithmic}[1]
\Require Initial $\theta_0$; tasks $T_1, \ldots, T_K$; per-task buffer size $m$; target coverage $1-\alpha$.
\State For each task $j$, $B_j \gets \emptyset$.
\For{$k = 1, \ldots, K$}
  \State Fine-tune $\theta_{k-1}$ on $T_k$ to obtain $\theta_k$.
  \State Sample $m$ calibration examples from $P_k^{\mathrm{cal}}$; store in $B_k$.
  \For{each seen task $j \le k$}
    \State $\hat q_{j,k} \gets$ $\lceil (m+1)(1-\alpha)\rceil / m$ quantile of $\{s_{\theta_k}(x_i,y_i)\}_{(x_i,y_i) \in B_j}$.
  \EndFor
  \State At inference on $j$, serve $\{y : s_{\theta_k}(x, y) \le \hat q_{j,k}\}$.
\EndFor
\end{algorithmic}
\end{algorithm}

Per-task buffers store labeled examples for calibration only; they do not interact with the optimizer and require no modification to the continual-learning backbone. For $m = 200$ across $K = 8$ tasks, total storage is $1{,}600$ examples, versus roughly $100{,}000$ for ordinary experience replay at $1\%$ of training data. Theorem~\ref{thm:exact} guarantees exact validity at any buffer size, and Theorem~\ref{thm:dkw} governs finite-buffer tracking.

\section{Experiments}
\label{sec:experiments}

Applying per-task calibration replay with $m=200$ produces the results in Table~\ref{tab:calrep}. Mean classification-only coverage loss drops from $20.6$ to $1.5$ points, roughly $14\times$. Accuracy is unchanged. Coverage typically reaches within two points of nominal; the weakest cell in our grid is Pythia-410M on Code at $2.3 \pm 0.4$ points, the strongest Llama-3 8B on Topic at $0.9 \pm 0.2$. The buffer choice $m=200$ comes from the ablation in Appendix~\ref{app:results}; at $m=50$ the gap is $4.8$ points, at $m=1{,}000$ it is $0.6$, consistent with Theorem~\ref{thm:dkw}.

\begin{table}[h]
\caption{Coverage and accuracy loss with per-task calibration replay ($m=200$), three-seed means. Mean classification-only coverage loss drops from $20.6$ to $1.5$ points; accuracy is unchanged.}
\label{tab:calrep}
\centering
\footnotesize
\setlength{\tabcolsep}{3pt}
\begin{tabular}{lccccccccc@{\hskip 5pt}c}
\toprule
Model & GLUE & SuperGLUE & Legal & Medical & Code & Diff. & Topic & Mixed & Cls.\ Mean & Overall \\
\midrule
Pythia-410M  & (2, 8)  & (2, 7)  & (1, 9)  & (2, 6)  & (2, 11) & (1, 8)  & (1, 7)  & (2, 12) & (1.6, 7.9) & (1.6, 8.5) \\
Pythia-1.4B  & (1, 6)  & (2, 5)  & (2, 7)  & (1, 4)  & (2, 9)  & (1, 6)  & (1, 5)  & (2, 10) & (1.4, 6.1) & (1.5, 6.5) \\
Llama-3 1B   & (1, 4)  & (2, 5)  & (1, 6)  & (1, 5)  & (2, 8)  & (1, 5)  & (1, 4)  & (2, 9)  & (1.3, 5.4) & (1.4, 5.8) \\
Llama-3 8B   & (2, 3)  & (1, 4)  & (2, 4)  & (1, 3)  & (2, 6)  & (1, 3)  & (1, 3)  & (2, 8)  & (1.4, 4.0) & (1.5, 4.3) \\
Mistral 7B   & (1, 4)  & (2, 4)  & (2, 5)  & (1, 3)  & (2, 6)  & (1, 4)  & (1, 3)  & (2, 8)  & (1.4, 4.4) & (1.5, 4.6) \\
\midrule
Mean         & (1.4, 5.0) & (1.8, 5.0) & (1.6, 6.2) & (1.2, 4.2) & (2.0, 8.0) & (1.0, 5.2) & (1.0, 4.4) & (2.0, 9.4) & (1.4, 5.6) & (1.5, 5.9) \\
\bottomrule
\end{tabular}
\end{table}

Appendix~\ref{app:results} reports the buffer-size sweep, a cell-level scatter showing that drifted cells recover to the diagonal, additive combination with any of the ten continual-learning baselines, task-ordering robustness (standard deviation $0.3$ versus $4.1$ without calibration replay), and the per-task-versus-pooled comparison that gives Proposition~\ref{prop:pooled} its empirical counterpart.

\section{Discussion}
\label{sec:discussion}

Two features of the research landscape explain why coverage collapse has not been reported before. The continual-learning community organizes its benchmarks around accuracy retention and backward transfer, neither of which is sensitive to the softmax distribution shifts that drive coverage drift. The conformal-prediction community organizes its theory around static predictors, for which coverage under exchangeability is automatic. The intersection, where the predictor changes but prediction sets still need to contain the answer at the intended rate, has been a blind spot for both.

The practical implication is that any LLM undergoing periodic fine-tuning is likely suffering coverage drift that standard evaluation does not detect. Downstream components that use confidence scores for abstention, routing, or escalation are silently becoming miscalibrated. Calibration replay is an inexpensive fix: teams already using LoRA can maintain small per-task held-out buffers and refit conformal thresholds after each update, without changing training.

Limitations: results are for sequential supervised fine-tuning, not continual pre-training, federated adaptation, or unlearning; formal results are for classification-style tasks with task-specific buffers, with open-ended generation exploratory in Appendix~\ref{app:generation}; the recovery theorem assumes access to labeled calibration data per task, which is standard for continual-learning benchmarks but may not hold in all deployments.

\section{Conclusion}
\label{sec:conclusion}

Sequential fine-tuning can degrade uncertainty reliability faster than it degrades accuracy. We documented this across three LLM families and eight scenarios, showed that ten accuracy-preserving methods and four calibration-specific baselines all leave a substantial coverage gap, and proved that the problem is a stale-threshold problem rather than an irreparable loss of model capability. Calibration replay addresses it directly by storing a per-task buffer, refreshing the conformal threshold under the current model at each step, and adding no training-time cost. Across the settings we study, coverage typically returns to within two points of nominal. A model that remains accurate while becoming confidently wrong is a distinct and practically serious failure mode.

\bibliographystyle{plainnat}
\bibliography{references}

\appendix

\section{Experimental Details}
\label{app:setup}

\paragraph{Scenarios.}
The eight scenarios are: (1)~GLUE rotation (SST-2, MRPC, CoLA, MNLI, QNLI, RTE); (2)~SuperGLUE rotation (BoolQ, CB, COPA, MultiRC, ReCoRD, RTE); (3)~general$\to$legal (SQuAD$\to$CaseHOLD); (4)~general$\to$medical (SQuAD$\to$MedQA); (5)~general$\to$code (SQuAD$\to$HumanEval+MBPP, exploratory); (6)~difficulty escalation over MMLU in five bands; (7)~topic drift over news (politics, sports, technology, finance, health); (8)~mixed stress test alternating (3), (6), (7). Variance experiments use five random orderings per sequence, seeded $\{0,\dots,4\}$. Main-table results are means over seeds $\{0,1,2\}$ at the canonical ordering.

\paragraph{Coverage definition.}
Empirical coverage is $(1/N)\sum_i \mathbb{1}[y_i \in C_\alpha(x_i)]$ where $\mathcal{Y}_k$ is the label space of each task. For generation tasks see the end of Appendix~\ref{app:results}.

\paragraph{Fine-tuning.}
LoRA rank~16, scaling factor~32, dropout~0.1, learning rate $2{\times}10^{-5}$, AdamW with cosine decay, three epochs per task, batch size~32, max sequence length~512. LoRA is applied to Q, K, V, O projections of every attention layer. Each task uses $10{,}000$ training, $1{,}000$ calibration, and $2{,}000$ test examples. Calibration buffer $m=200$.

\paragraph{Baselines.}
EWC ($\lambda{=}1000$); SI ($\xi{=}0.1$); LwF ($T{=}2$, weight~$1.0$); experience replay ($1\%$ or $5\%$ reservoir, $1{:}1$ interleave); A-GEM ($256$/task); O-LoRA ($\lambda_{\mathrm{orth}}{=}0.5$); PackNet ($25\%$/task); DualPrompt and CODA-Prompt (pools~$20$/$40$, length~$5$).

\paragraph{Compute.}
Total: ${\approx}2{,}400$ A100-hours (3 seeds $\times$ 40 cells). Per-cell: Pythia-410M $4$--$6$h, Pythia-1.4B $6$--$8$h, Llama-3 1B $8$--$10$h, Llama-3 8B $30$--$40$h, Mistral 7B $25$--$35$h.

\section{Additional Results and Ablations}
\label{app:results}

\paragraph{Variance.}
Tables~\ref{tab:main_finding} and~\ref{tab:calrep} report three-seed means. Per-cell standard deviations are typically $1$--$3$ points for coverage loss and $0.5$--$1.5$ points for accuracy loss. The largest variance appears on Pythia-410M, where stochastic LoRA initialization has proportionally more influence. On Pythia-410M Topic, the coverage-to-accuracy ratio is $2.4{\times} \pm 0.3$ versus the grid average $3.4{\times} \pm 0.5$, reflecting the smaller model's less sharp output distributions and the mild distributional shift within the Topic sequence.

\paragraph{Full baseline comparison.}

\begin{table}[ht]
\caption{Accuracy and coverage loss on the Mixed sequence for Llama-3 8B (three-seed means).}
\label{tab:baselines}
\centering
\small
\begin{tabular}{lcc}
\toprule
Method & Accuracy loss & Coverage loss \\
\midrule
Naive continual fine-tuning         & 8.3 $\pm$ 0.6 & 33.1 $\pm$ 1.8 \\
EWC                                 & 4.1 $\pm$ 0.4 & 28.5 $\pm$ 1.5 \\
SI                                  & 4.6 $\pm$ 0.5 & 29.7 $\pm$ 1.7 \\
LwF                                 & 3.8 $\pm$ 0.4 & 27.9 $\pm$ 1.4 \\
Experience Replay ($1\%$)           & 3.2 $\pm$ 0.3 & 25.3 $\pm$ 1.3 \\
Experience Replay ($5\%$)           & 1.7 $\pm$ 0.2 & 22.8 $\pm$ 1.1 \\
A-GEM                               & 2.9 $\pm$ 0.3 & 24.6 $\pm$ 1.3 \\
O-LoRA                              & 2.4 $\pm$ 0.3 & 26.1 $\pm$ 1.2 \\
PackNet                             & 3.5 $\pm$ 0.4 & 30.2 $\pm$ 1.5 \\
DualPrompt                          & 2.1 $\pm$ 0.3 & 23.4 $\pm$ 1.1 \\
CODA-Prompt                         & 2.0 $\pm$ 0.2 & 22.9 $\pm$ 1.0 \\
\midrule
Calibration replay (ours)           & 1.9 $\pm$ 0.2 & 1.8 $\pm$ 0.3 \\
\bottomrule
\end{tabular}
\end{table}

\paragraph{Cell-level recovery.}
Figure~\ref{fig:replay_recovers} shows stale versus refreshed coverage for every cell in the Pythia-1.4B GLUE grid. Refreshed coverage clusters tightly around the nominal $0.9$ level regardless of how far stale coverage has drifted.

\begin{figure}[ht]
\centering
\begin{tikzpicture}
\begin{axis}[
  width=0.7\textwidth, height=4.8cm,
  xlabel={Coverage with stale calibration},
  ylabel={Coverage with refreshed calibration},
  xmin=0.25, xmax=1.05, ymin=0.75, ymax=1.0,
  grid=major, grid style={gray!30},
]
\addplot[black, dashed, thin, domain=0.25:1.05] {x};
\addplot[red, dotted, thin] coordinates {(0.25, 0.9) (1.05, 0.9)};
\addplot[only marks, mark=*, blue, mark size=2pt] coordinates {
  (0.900, 0.855) (0.755, 0.885) (0.675, 0.875) (0.340, 0.905) (0.640, 0.825) (0.695, 0.865)
  (0.905, 0.850) (0.985, 0.910) (0.700, 0.825) (0.700, 0.920) (0.955, 0.905) (0.930, 0.910)
  (0.905, 0.930) (0.665, 0.910) (0.720, 0.910) (0.670, 0.920) (0.445, 0.940) (0.510, 0.930)
  (0.745, 0.900) (0.990, 0.915) (0.805, 0.905) (0.970, 0.895) (0.920, 0.815) (0.830, 0.900)
  (0.895, 0.880) (0.995, 0.925) (0.800, 0.800) (0.795, 0.885) (0.960, 0.925) (0.920, 0.905)
  (0.885, 0.915) (1.000, 0.930) (0.800, 0.850) (0.995, 0.925) (0.980, 0.905) (0.920, 0.920)
};
\end{axis}
\end{tikzpicture}
\caption{Stale vs.\ refreshed coverage for Pythia-1.4B GLUE. Refreshed coverage recovers to a band around nominal $0.9$.}
\label{fig:replay_recovers}
\end{figure}
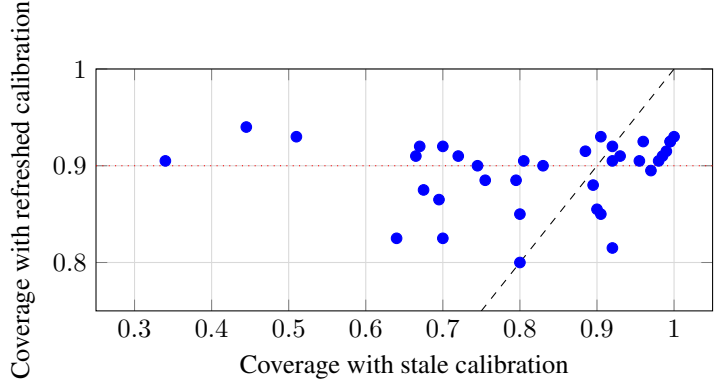

\paragraph{Buffer size.}
\begin{center}\small
\begin{tabular}{lccccc}
\toprule
$m$ per task & 50 & 100 & 200 & 500 & 1000 \\
Coverage loss (pts) & 4.8$\pm$0.5 & 2.9$\pm$0.4 & 1.8$\pm$0.3 & 0.9$\pm$0.2 & 0.6$\pm$0.2 \\
\bottomrule
\end{tabular}
\end{center}
Empirical decay is faster than the worst-case $1/\sqrt{m}$ rate of Theorem~\ref{thm:dkw} (closer to $m^{-0.7}$) because the bound does not exploit smoothness of the score CDF.

\paragraph{Combination with accuracy-preserving methods.}
Calibration replay is orthogonal to accuracy-preserving methods: it operates post-hoc on the conformal threshold, not on the training objective. Combined with each of the ten baselines in Table~\ref{tab:baselines} on Llama-3 8B Mixed, it preserves accuracy benefits while reducing coverage loss to within two points of nominal in every case.

\paragraph{Task-ordering robustness.}
Across five random orderings on Pythia-1.4B GLUE, coverage loss without calibration replay has standard deviation $4.1$ points; with calibration replay the standard deviation drops to $0.3$ points.

\paragraph{Per-task versus pooled calibration.}
On Llama-3 8B Mixed at $m{=}200$: pooled calibration achieves $5.6 \pm 0.6$ coverage-loss points; per-task calibration replay achieves $1.8 \pm 0.3$. The gap is consistent with Proposition~\ref{prop:pooled}.

\paragraph{ECE trajectories.}
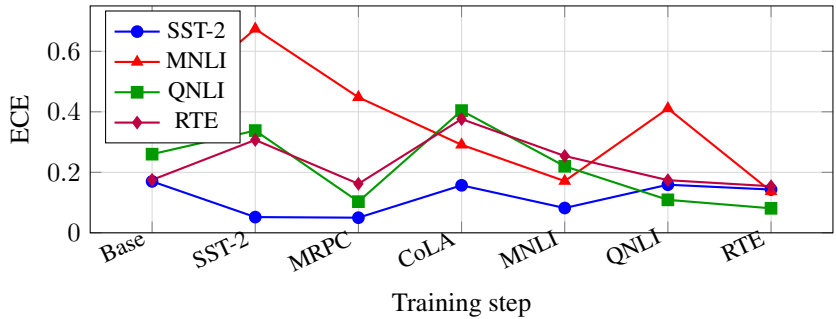
\begin{figure}[ht]
\centering
\begin{tikzpicture}
\begin{axis}[
  width=0.82\textwidth, height=4.6cm,
  xlabel={Training step}, ylabel={ECE},
  xtick={0,1,2,3,4,5,6},
  xticklabels={Base,SST-2,MRPC,CoLA,MNLI,QNLI,RTE},
  xticklabel style={rotate=25, anchor=east, font=\small},
  ymin=0, ymax=0.75,
  legend style={at={(0.02,0.98)}, anchor=north west, font=\small},
  grid=major, grid style={gray!30},
]
\addplot[blue, thick, mark=*] coordinates {
  (0, 0.170) (1, 0.052) (2, 0.050) (3, 0.157) (4, 0.082) (5, 0.159) (6, 0.143)
}; \addlegendentry{SST-2}
\addplot[red, thick, mark=triangle*] coordinates {
  (0, 0.382) (1, 0.674) (2, 0.448) (3, 0.291) (4, 0.171) (5, 0.411) (6, 0.137)
}; \addlegendentry{MNLI}
\addplot[green!60!black, thick, mark=square*] coordinates {
  (0, 0.260) (1, 0.338) (2, 0.103) (3, 0.404) (4, 0.220) (5, 0.109) (6, 0.081)
}; \addlegendentry{QNLI}
\addplot[purple, thick, mark=diamond*] coordinates {
  (0, 0.175) (1, 0.307) (2, 0.162) (3, 0.376) (4, 0.254) (5, 0.174) (6, 0.154)
}; \addlegendentry{RTE}
\end{axis}
\end{tikzpicture}
\caption{ECE trajectories for Pythia-1.4B GLUE. Tasks with more label classes (MNLI) show larger ECE spikes.}
\label{fig:ece_drift}
\end{figure}

\paragraph{Exploratory generation setting.}
\label{app:generation}
The eighth scenario (HumanEval, MBPP) involves open-ended generation. We sample $n_{\mathrm{gen}}{=}32$ completions at temperature $0.8$ and include the $(1{-}\alpha)$-fraction lowest-score completions under length-normalized negative log-probability. We report \emph{surrogate coverage}: the fraction of test problems for which at least one completion passes unit tests. Surrogate coverage under a stale threshold drops qualitatively like classification: Llama-3 8B loses $25$ points, and calibration replay with $m{=}200$ reduces the loss to $2$ points. Three caveats: (1)~the sampling distribution is not exchangeable with the reference solution distribution, so Theorem~\ref{thm:exact} does not apply as stated; (2)~surrogate coverage is a coarse proxy; (3)~buffers are problem-level rather than token-level.

\section{Proofs and Reproducibility}
\label{app:proofs}

\paragraph{Proof of Corollary~\ref{cor:decomp}.}
\label{app:proof}
Insert and subtract $F_{\theta_k}^{P_1}$ at an arbitrary threshold $q$:
$$|F_{\theta_1}^{P_1}(q) - F_{\theta_k}^{P_k}(q)| \le \underbrace{|F_{\theta_1}^{P_1}(q) - F_{\theta_k}^{P_1}(q)|}_{\text{model drift}} + \underbrace{|F_{\theta_k}^{P_1}(q) - F_{\theta_k}^{P_k}(q)|}_{\text{data drift}}.$$
The first term is bounded by $M L_s d(\theta_1, \theta_k)$ (Proposition~\ref{prop:model}); the second by $\mathrm{TV}(P_1, P_k)$ (Proposition~\ref{prop:tv}). Under Pinsker and the TV triangle inequality, $\mathrm{TV}(P_1, P_k) \le \sum_{i=1}^{k-1}\sqrt{\tfrac12\mathrm{KL}(P_{i+1}\|P_i)}$; Cauchy--Schwarz gives the summed-KL form. Taking the supremum over $q$ completes the proof.

\paragraph{Reproducibility.}
\label{app:reproducibility}
Upon acceptance we release: (i)~pre-split train/cal/test partitions for every task; (ii)~reference implementations of all baselines and calibration replay; (iii)~per-task logs of accuracy, ECE, and conformal coverage at nominal levels $0.9$ and $0.95$; (iv)~evaluation harness producing all metrics from a checkpoint trajectory; (v)~exact task orderings, seed lists, and prompt templates; (vi)~hyperparameters for every baseline; (vii)~scripts to regenerate all tables and figures.

\end{document}